\PassOptionsToPackage{table}{xcolor}
\documentclass[11pt]{article}
\usepackage{arxiv}

\usepackage{latexsym}
\usepackage[T1]{fontenc}
\usepackage[utf8]{inputenc}
\usepackage{microtype}
\usepackage{inconsolata}
\usepackage{natbib}

\usepackage{amsmath,amsfonts,bm}
\usepackage{xcolor}

\def\1{\bm{1}}

\def\cR{{\mathcal{R}}}

\DeclareMathAlphabet{\mathsfit}{\encodingdefault}{\sfdefault}{m}{sl}
\SetMathAlphabet{\mathsfit}{bold}{\encodingdefault}{\sfdefault}{bx}{n}

\def\Pr{\mathbb{P}}

\makeatletter
\def\thickhline{%
  \noalign{\ifnum0=`}\fi\hrule \@height \thickarrayrulewidth \futurelet
   \reserved@a\@xthickhline}
\def\@xthickhline{\ifx\reserved@a\thickhline
               \vskip\doublerulesep
               \vskip-\thickarrayrulewidth
             \fi
      \ifnum0=`{\fi}}
\makeatother

\newlength{\thickarrayrulewidth}
\NewDocumentCommand{\yl}{ mO{} }{\textcolor{brown}{\textsuperscript{\textit{YL}}\textrm{{\small[#1]}}}}

\NewDocumentCommand{\nk}{ mO{} }{\textcolor{orange}
{\textsuperscript{\textit{NK}}\textrm{{\small[#1]}}}}

\NewDocumentCommand{\nikki}{ mO{} }{\textcolor{purple}{\textsuperscript{\textit{Nikki}}\textrm{{\small[#1]}}}}

\usepackage{graphbox}

\def\cR{\mathcal{R}}

\usepackage{booktabs}
\usepackage{tabularx}

\usepackage[utf8]{inputenc}
\usepackage{amsmath}       %
\usepackage{amssymb}       %
\usepackage{xcolor}        %
\usepackage{mdframed}      %
\usepackage{bbm}           %

\usepackage{enumitem}

\usepackage{spverbatim}

\newtheorem{theorem}{Theorem}[section]
\newtheorem{proposition}[theorem]{Proposition}

\newtheorem{definition}[theorem]{Definition}

\usepackage{xcolor}

\usepackage{tikz}
\usepackage{pgfplots}
\usetikzlibrary{arrows.meta,positioning,calc}
\pgfplotsset{compat=1.18}

\usepackage{algorithm}
\usepackage{algorithmic}

\usepackage{multirow}

\usepackage{tikz}
\usetikzlibrary{arrows.meta,positioning,matrix,fit,calc}

\newenvironment{proof}[1][Proof]{\begin{trivlist}
\item[\hskip \labelsep {\bfseries #1}]}{\qed\end{trivlist}}

\usepackage{graphicx}
\usepackage{array}
\usepackage{xcolor}
\usepackage{tikz}
\usetikzlibrary{positioning,calc,arrows.meta}

\definecolor{headerblue}{RGB}{47,95,170}
\definecolor{arrowblue}{RGB}{35,50,85}
\definecolor{lightrow}{RGB}{242,245,250}
\definecolor{goldfill}{RGB}{245,210,90}
\definecolor{goldtext}{RGB}{160,60,20}
\definecolor{clusterblue}{RGB}{210,220,235}

\definecolor{rankdark}{RGB}{46,125,50}
\definecolor{rankmid}{RGB}{102,187,106}
\definecolor{ranklight}{RGB}{165,214,167}
\definecolor{rankvlight}{RGB}{220,237,200}
\definecolor{rankbad}{RGB}{235,235,235}

\usepackage{graphicx}
\usepackage{array}
\usepackage{xcolor}
\usepackage{tikz}
\usetikzlibrary{positioning,calc,fit,arrows.meta}

\definecolor{arrowblue}{RGB}{35,50,85}
\definecolor{rankdark}{RGB}{46,125,50}
\definecolor{rankmid}{RGB}{102,187,106}
\definecolor{ranklight}{RGB}{165,214,167}
\definecolor{rankvlight}{RGB}{220,237,200}
\definecolor{rankbad}{RGB}{235,235,235}
\definecolor{goldfill}{RGB}{245,210,90}
\definecolor{goldtext}{RGB}{160,60,20}

\usepackage{graphicx}
\usepackage{array}
\usepackage{xcolor}
\usepackage{tikz}
\usetikzlibrary{positioning,calc,arrows.meta}

\definecolor{arrowblue}{RGB}{35,50,85}
\definecolor{rankdark}{RGB}{46,125,50}
\definecolor{rankmid}{RGB}{102,187,106}
\definecolor{ranklight}{RGB}{165,214,167}
\definecolor{rankvlight}{RGB}{220,237,200}
\definecolor{rankbad}{RGB}{245,245,245}
\definecolor{goldfill}{RGB}{245,210,90}
\definecolor{goldtext}{RGB}{160,60,20}

\definecolor{rankdark}{RGB}{94,145,98}
\definecolor{rankmid}{RGB}{137,182,140}
\definecolor{ranklight}{RGB}{176,209,178}
\definecolor{rankvlight}{RGB}{218,233,219}
\definecolor{rankbad}{RGB}{242,242,242}

\usepackage{pgfplots}
\pgfplotsset{compat=1.18}

\definecolor{sagegreen}{RGB}{122,140,124}

\usepackage[T1]{fontenc}
\usepackage{textcomp}
\usepackage{upquote}
\usepackage{xcolor}
\usepackage{listings}
\usepackage{mdframed}
\usepackage{enumitem}

\lstdefinestyle{promptpython}{
  language=Python,
  basicstyle=\ttfamily\footnotesize,
  keywordstyle=\color{blue!60!black},
  commentstyle=\color{green!40!black},
  stringstyle=\color{orange!50!black},
  showstringspaces=false,
  breaklines=true,
  frame=single,
  rulecolor=\color{black!35},
  framerule=0.5pt,
  backgroundcolor=\color{black!2},
  tabsize=2,
  xleftmargin=4pt,
  xrightmargin=4pt,
  framexleftmargin=4pt,
  framexrightmargin=4pt,
  framextopmargin=4pt,
  framexbottommargin=4pt,
  upquote=true
}

\definecolor{appendixframe}{RGB}{150,150,150}
\definecolor{appendixfill}{RGB}{248,248,248}

\newmdenv[
  linecolor=appendixframe,
  backgroundcolor=appendixfill,
  linewidth=0.6pt,
  roundcorner=2pt,
  skipabove=0.5\baselineskip,
  skipbelow=0.8\baselineskip,
  innerleftmargin=8pt,
  innerrightmargin=8pt,
  innertopmargin=7pt,
  innerbottommargin=7pt,
  splittopskip=7pt,
  splitbottomskip=7pt
]{promptbox}

\lstdefinestyle{appendixcode}{
  language=Python,
  basicstyle=\ttfamily\scriptsize,
  keywordstyle=\color{blue!60!black},
  commentstyle=\color{green!40!black},
  stringstyle=\color{orange!50!black},
  breaklines=true,
  columns=fullflexible,
  keepspaces=true,
  showstringspaces=false,
  upquote=true,
  frame=single,
  rulecolor=\color{appendixframe},
  framerule=0.6pt,
  backgroundcolor=\color{appendixfill},
  xleftmargin=4pt,
  xrightmargin=4pt,
  framexleftmargin=4pt,
  framexrightmargin=4pt,
  framextopmargin=4pt,
  framexbottommargin=4pt,
  aboveskip=0.7\baselineskip,
  belowskip=0.7\baselineskip
}

\usepackage{xcolor}
\usepackage{booktabs}
\usepackage{multirow}
\usepackage{colortbl}

\usepackage{array}

\usepackage{graphicx}
\usepackage{booktabs}
\usepackage{multirow}
\usepackage{array}
\usepackage{xcolor}

\def\code{\mathsf{c}}

\providecommand{\qed}{\hfill\ensuremath{\square}}
\usepackage{url}
\usepackage{booktabs}
\usepackage{hyperref}

\definecolor{darkblue}{rgb}{0, 0, 0.5}
\hypersetup{colorlinks=true, citecolor=darkblue, linkcolor=darkblue, urlcolor=darkblue}

\title{Reinforcement Learning without Ground-Truth Solutions can Improve LLMs}

\renewcommand{\undertitle}{}
\date{}

\author{%
\textbf{Yingyu Lin}\({}^{1,2}\)\thanks{Equal contribution.}\thanks{Work done while interning at Snowflake AI Research.} \quad
\textbf{Qiyue Gao}\({}^{1}\)\footnotemark[1] \quad
\textbf{Nikki Lijing Kuang}\({}^{2}\)\footnotemark[1] \quad
\textbf{Xunpeng Huang}\({}^{1}\) \quad
\textbf{Kun Zhou}\({}^{1}\)\\
\textbf{Tongtong Liang}\({}^{1}\) \quad
\textbf{Zhewei Yao}\({}^{2}\) \quad
\textbf{Yi-An Ma}\({}^{1}\)\thanks{Co-senior authors.} \quad
\textbf{Yuxiong He}\({}^{2}\)\footnotemark[3]\\[0.35em]
\({}^{1}\)University of California, San Diego \quad
\({}^{2}\)Snowflake AI Research
}

\begin{document}

\maketitle

\begin{abstract}

Reinforcement Learning with Verifiable Rewards (RLVR) for training LLMs typically rely on ground-truth answers to assign rewards, limiting their applicability to tasks where the ground-truth solution is unknown. We introduce a \textbf{R}anking-\textbf{i}nduced \textbf{VER}ifiable framework (RiVER) that trains LLMs on score-based optimization tasks without ground-truth solutions, using deterministic execution feedback as continuous-valued supervision. When applying group-relative RL to such continuous rewards, we identify two key challenges: \emph{scale dominance}, where uncalibrated score magnitudes across test instances distort policy updates, and \emph{frequency dominance}, where repeatedly sampled suboptimal solutions can outweigh rare but stronger candidates. RiVER addresses these challenges with calibrated reward shaping that uses instance-wise comparisons and emphasizes top-ranked solvers while retaining bounded feedback for other valid solutions. We train on 12 AtCoder Heuristic Contest tasks and evaluate on Algorithm Engineering Benchmark (ALE-Bench), LiveCodeBench, and USACO. RiVER advances Qwen3-8B and GLM-Z1-9B-0414 by 8.9\% and 9.4\% in ALE rating rank. More importantly, despite training exclusively on score-based tasks without any ground-truth solutions, RiVER also improves the backbones across exact-solution benchmarks such as LiveCodeBench and USACO by an absolute average improvement of 2.4\% and 3.5\%. By contrast, baselines trained with raw execution scores improve ALE rating but fail to transfer to exact-solution benchmarks. These results suggest that score-based optimization tasks, combined with proper reward calibration, can serve as effective training environments for general coding ability without ground-truth solutions.
\end{abstract}

\begin{figure}[t]
    \centering
    \includegraphics[width=0.99\textwidth]{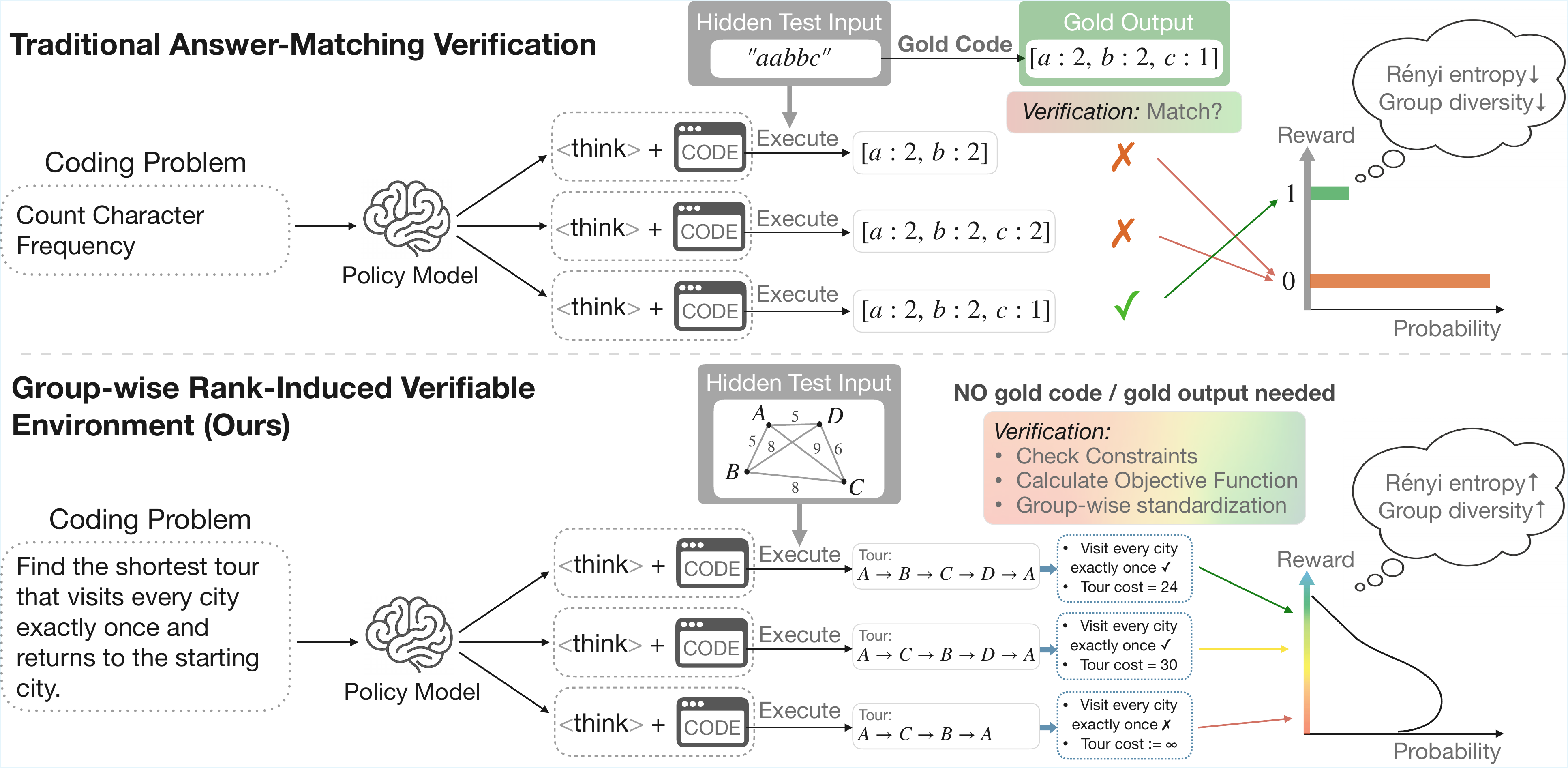}
    \caption{Comparison between traditional answer-matching verification and our group-wise rank-induced verifiable environment. Traditional verification relies on gold code or gold outputs and yields sparse binary rewards, while our method assigns relative rewards by comparing candidate solutions through constraint checking and objective evaluation, without requiring ground-truth solutions.}
    \label{fig:placeholder}
\end{figure}

\section{Introduction}

Reinforcement Learning with Verifiable Rewards (RLVR) has become a central recipe for improving LLM reasoning~\citep{wei2022chain, lambert2024tulu, wang2025octothinker, kumar2025llm, zhang2025interplay}. In mathematics and coding, this paradigm has enabled strong post-training results by rewarding outputs that match a known answer or pass unit tests. However, the dominant form of RLVR remains closely tied to answer matching. A generated response is typically rewarded if it satisfies a \textit{binary} correctness criterion. This makes RLVR powerful but narrow. Many important reasoning and coding tasks do not have a single ground-truth solution. In algorithm design, optimization, and planning, there may be many feasible solutions with different qualities, and the optimal solution may be unknown or computationally intractable. Such tasks fall outside the standard answer-matching formulation, even though they are often still verifiable in a broader sense: candidate solutions can be executed, checked for feasibility, and \textit{compared} by an objective function.

This paper studies whether ground-truth-free verification can provide an effective training signal for LLMs. We focus on open-ended, score-based algorithm-engineering contests, such as AtCoder Heuristic Contests~\citep{atcoder} and Topcoder Marathon Matches~\citep{topcoder_marathon_match_tournament}. In these tasks, the policy model must write programs that produce high-scoring approximate solutions to complex optimization problems. Unlike answer-matching RLVR, the environment provides no reference program, gold output, or certified optimum. Instead, multiple sampled solvers are executed on the same hidden instances; invalid outputs are rejected by constraint checks, and feasible outputs are evaluated by a task-specific objective function. The training signal comes from asking which sampled solver performs better under execution, not whether it matches a precomputed solution. This turns optimization tasks into a scalable source of verifiable supervision: the model need not know an optimal solution, as long as the environment can check feasibility and compare sampled candidates.

Naively using these continuous-valued objective scores for reinforcement learning, however, is not sufficient. Compared with binary correctness rewards, score-based feedback is more fine-grained and potentially more informative, since it can distinguish different levels of solution quality among feasible candidates. Yet this additional information is not automatically suitable for policy-gradient optimization. We identify two key challenges. First, raw objective scores are often uncalibrated across hidden instances. Even within the same optimization problem, score magnitudes can vary with instance size, graph structure, weight scale, or outliers. Aggregating raw scores across instances can therefore make policy updates depend on arbitrary numerical scales rather than robust improvements in solution quality. We refer to this effect as scale dominance: instances with larger score ranges dominate the reward signal regardless of whether they provide more meaningful learning information. Second, group-relative learning can underemphasize rare high-quality discoveries. In open-ended optimization tasks, a rollout group often contains multiple syntactic or parametric variants of the same feasible heuristic algorithm. Because group-relative objectives assign credit at the level of individual samples, these repeated non-winning variants can collectively contribute more gradient mass than a stronger solver that is discovered only once. As a result, the update reflects not only relative solution quality, but also the sampling frequency of different behavioral modes. We refer to this mismatch as frequency dominance.

We propose RiVER, a Ranking-induced VERifiable reinforcement learning framework for training LLMs from executable optimization tasks without ground-truth solutions. RiVER converts uncalibrated execution feedback into stable policy-gradient signals through two simple design choices. First, it performs instance-wise ranking: for each hidden test instance, candidate solvers are compared only against other candidates evaluated on the same instance. This removes arbitrary score-scale effects while preserving relative solution quality. Second, RiVER applies winner-heavy reward shaping: the best candidate in the group is separated from non-winning candidates, while valid non-winners still receive bounded non-binary feedback. This reduces the influence of repeated suboptimal modes and focuses the update on the strongest solver discovered within the sampled group.

This formulation extends RLVR beyond the answer-matching regime. Instead of asking whether a response equals a known solution, RiVER asks whether one generated solution is better than another under an executable evaluator. The resulting supervision is ground-truth-free, non-binary, and verifiable: it requires no human labels, no reference solution, and no optimal answer, yet still provides fine-grained comparisons among sampled candidates. More broadly, it suggests that optimization problems can be used not only as target benchmarks, but also as training environments for improving general reasoning and coding abilities.

We instantiate RiVER on AtCoder Heuristic Contest tasks and evaluate its effects on both score-based and exact-solution programming benchmarks. Training is performed only on open-ended optimization tasks without ground-truth solutions. Nevertheless, RiVER improves performance not only on score-based algorithm engineering evaluation, but also on conventional pass/fail coding benchmarks such as LiveCodeBench and USACO. In contrast, baselines that optimize raw execution scores improve score-based ratings but fail to transfer consistently to exact-solution benchmarks. These results indicate that the benefit does not come merely from exposing the model to more executable feedback; rather, proper reward calibration is crucial for turning score-based environments into transferable supervision.


\section{Related Work}

\paragraph{Reinforcement Learning for LLM Reasoning} 
Incentivizing LLMs to reason explicitly through multi-step process have given rise to large reasoning models(LRMs) that excel on challenging tasks including mathematics \citep{shao2024deepseekmath, guo2025deepseek}, coding \citep{huang2025opencoder, liu2025rstar, hui2024qwen2}, agentic systems \citep{wu2025webdancer, li2025websailor}. 
For tasks with verifiable outcomes, RLVR is a prominent paradigm, optimizing model outputs with deterministic rule-based correctness signals. However, prevailing methods \citep{schulman2017proximal, shao2024deepseekmath, yu2025dapo} heavily depends on large volumes of in-domain data, limiting applicability in regimes where labeled data is insufficient or expensive to collect.

\paragraph{Reward Design and Exploration}

Long-horizon reasoning in LLMs introduces a challenge for reward design in RL due to the sparsity and low informativeness of outcome-level supervision. Binary outcome rewards provide limited guidance for attributing the final result to individual reasoning steps in multi-step CoT trajectories and frequently induce premature convergence to similar trajectories, resulting in diversity collapse \citep{yao2025diversity, cheng2025reasoning}. Prior work has explored step-level supervision through process reward models (PRMs) \citep{setlur2024rewarding, zhang2025lessons, she2025r, zhang2024entropy}, improved credit assignment mechanisms in single-turn \citep{wang2024math, jiao2024learning, lightman2023let, li2025attention} and multi-turn \citep{wang2025steca, feng2025group, xiong2024watch} settings, as well as entropy-based objectives to encourage exploration \citep{wang2025beyond, cui2025entropy, cheng2025reasoning}. Recent work has also investigated improving learning signal quality through adaptive training environments, where task difficulty is dynamically adjusted to match model capability in verifiable settings \citep{zeng2025rlve}.  However, these approaches either rely on costly step-level labels or leave the reward signal coarse. Furthermore, while relative ranking rewards have been widely studied in reinforcement learning from human feedback (RLHF), where models are optimized using pairwise or listwise comparisons between responses, how to incorporate such idea into RLVR is underexplored \citep{rafailov2023direct, li2024process}. We instead adopt dense relative ranking rewards with a co-evolutionary training framework, enabling fine-grained supervision to consistently learn from solution quality differences without human labels while preserving exploration and diverse reasoning behaviors.

\paragraph{Reasoning via Optimization} 
In contrast to standard reasoning benchmarks such as question answering or mathematical problem solving, optimization tasks requires models to produce feasible solutions while simultaneously maximizing solution optimality \citep{yang2025nondeterministic, fan2024nphardeval}. Owing to the computational complexity and evaluation challenges, this class of problems has remained relatively underexplored, especially in training-based settings of LLM reasoning. Recent efforts have been made to employ LLMs as end-to-end training-based solvers for NP-hard problems \citep{jiang2024llmopt, li2025np, jiang2025large, wang2025graph}, learning heuristics to improve solution quality or feasibility \citep{yang2025heuragenix, chen2025heurigym, ye2024reevo, wu2025efficient},  and integrating LLMs into existing optimization pipelines at inference time \citep{li2025strcmp}. These methods typically treat optimization as the target task and focus on task-specific evaluation metrics. In contrast, our work uses NP-hard optimization problems as a training environment for general reasoning rather than an end goal. The algorithmic structure and continuous objective signals inherent in optimization naturally provide dense learning signals for RL training, and eventually lead to transferable improvements on out-of-domain reasoning tasks. 

\section{Preliminaries}

\paragraph{Verifiable Rewards} Let $\pi_{\theta}$ be a parameterized LLM  policy that generates a sequence of tokens
$\mathbf{o}= (o_{1}, \dots, o_{T})$ conditioned on a prompt question $q$, where $q$ sampled uniformly from the training corpus $\mathcal{D}$.
RLVR assumes access to an outcome-level reward function $\mathcal{R}: \mathcal{Q} \times \mathcal{O} \rightarrow \mathbb{R}$,
which assigns a scalar reward to each prompt--response pair $(q,\mathbf{o})$. In standard verifiable settings, the output of $\mathcal{R}$ is typically binary, i.e., $\mathcal{R}(q,\mathbf{o}) \in \{0,1\}$ indicating whether the generated response satisfies a task-specific verifier such as correctness, successful code execution, or a formal constraint. 
The objective of RLVR is to learn a policy that maximizes the expected reward:%
\begin{equation}
        \max_\theta J(\theta) =
    \mathbb{E}_{q \sim \mathcal{D}}
    \left[
    \mathbb{E}_{\mathbf{o} \sim \pi_\theta(\cdot | q)} [\, \cR(q, \mathbf{o}) \,]
    \right].
\end{equation}
Its gradient can be written as
\begin{equation}
\nabla_{\theta}J(\theta)
=
\mathbb{E}_{q\sim\mathcal{D},\mathbf{o}\sim\pi_{\theta}(\cdot\mid q)}
\left[
\mathcal{R}(q,\mathbf{o})
\nabla_{\theta}\log \pi_{\theta}(\mathbf{o}\mid q)
\right].
\label{eq:policy_gradient}
\end{equation}

\paragraph{RLVR Algorithms}
A commonly used RLVR algorithm is Group Relative Policy Optimization (GRPO)~\citep{shao2024deepseekmath}, a variant of Proximal Policy Optimization (PPO)~\citep{schulman2017proximal}. GRPO removes the need for a learned value model by estimating advantages from the relative rewards of a group of sampled responses. For each prompt $q$, GRPO samples $G$ responses $\{\mathbf{o}_i\}_{i=1}^{G}$ from the old policy $\pi_{\theta_{\mathrm{old}}}$ and assigns each response a group-relative advantage $\hat{A}_i$ based on its reward within the group. In practice, $\hat{A}_i$ is often obtained by normalizing the response-level rewards across the $G$ samples. Its clipped surrogate objective is
\begin{equation}\label{eq:grpo_loss}
\begin{aligned}
    \mathcal{J}_\text{GRPO}(\theta) &= \mathbb{E}_{q \sim \mathcal{D}, \{\mathbf{o}_i\}_{i=1}^G\sim \pi_{\theta_\text{old}}(\cdot\mid q)} 
     \Bigg[ \frac{1}{G}\sum_{i=1}^{G} \frac{1}{|\mathbf{o}_i|}\sum_{t=1}^{|\mathbf{o}_i|} \Big(\min \Big( r_{i,t}(\theta) \hat{A}_{i}, 
      \text{clip} \Big( r_{i,t}(\theta), 1 - \varepsilon, 1 + \varepsilon \Big) \hat{A}_{i} \Big) 
    - \beta D_{\text{KL}}(\pi_{\theta} || \pi_{\text{ref}}) \Big) \Bigg],
\end{aligned}
\end{equation}
where
$
    r_{i,t}(\theta):=\frac{\pi_{\theta}(o_{i,t} \mid q, o_{i,<t})}{\pi_{\theta_{\text{old}}}(o_{i,t} \mid q,o_{i,<t})}
$
is the importance sampling ratio, $\varepsilon$ controls the clipping range, $\beta$ determines the strength of
KL regularization, and $\pi_{\mathrm{ref}}$ denotes a fixed reference policy.

In this paper, we build on the GRPO framework but replace standard binary or raw-score rewards with rank-induced, winner-heavy rewards derived from executable optimization environments. This allows the policy to learn from verifiable comparisons among sampled solutions even when no ground-truth answer is available.

\paragraph{Optimization and NP-hardness.}
An optimization problem asks for a feasible solution that maximizes a task-specific objective, or equivalently minimizes a cost. Unlike exact-answer problems, an optimization task may admit many valid solutions with different qualities, and the globally optimal solution may be unknown. Many such problems are NP-hard, where NP stands for nondeterministic polynomial time: an NP-hard problem is at least as hard as every problem in NP, in the sense that a polynomial-time exact algorithm for it would imply polynomial-time algorithms for all NP problems. In practice, they are often solved by heuristic algorithms, which aim to produce high-quality feasible solutions within limited time but do not guarantee optimality.
Although finding an optimal solution may be intractable, checking whether a candidate output is feasible and computing its objective value are typically efficient. Therefore, heuristic algorithms can be evaluated and compared by their achieved objective scores, even when no optimal ground-truth solution is available.

\section{Method}
\label{sec:method}

This section describes how RiVER turns ground-truth-free optimization tasks into reinforcement learning signals. For each problem, the policy samples a group of executable solvers. Each solver is run on the same hidden test instances and evaluated by a deterministic task-specific evaluator. The evaluator does not provide a gold program, gold output, or optimal solution; it only checks feasibility and returns an objective value for valid outputs. RiVER then converts these objective values into instance-wise ranks and applies winner-heavy reward shaping before optimizing the policy with GRPO.

\subsection{Ground-truth-free Executable Optimization}
\label{sec:training_task}

A training task consists of a problem prompt $q$, a deterministic evaluator $\mathcal{E}$, and a set of hidden test instances $\{T_{m}\}_{m=1}^{M}$. The model generates a response $\mathbf{o}$, which may include intermediate reasoning, but is required to end with a complete executable program. We extract this final program as $\code=\operatorname{Code}(\mathbf{o})$ and execute only $\code$ during evaluation. Running $\code$ on instance $T_{m}$ produces either a failure or a valid output with a task-specific objective value.

We denote the validity indicator by $v(\code,T_{m})\in\{0,1\}$, where $v(\code,T_{m})=0$ if the program crashes, times out, violates output constraints, or produces an invalid format. If $v(\code,T_{m})=1$, the evaluator returns a task-specific objective value. We denote the corresponding objective score by $f(\code,T_{m})$, using a unified larger-is-better convention: for maximization tasks, $f(\code,T_{m})$ is the objective value itself; for minimization tasks, $f(\code,T_{m})$ is the negated objective value. Invalid executions are handled separately in reward shaping.

For each prompt $q$, we sample a group of $G$ candidate responses from the rollout policy, $\mathbf{o}_{1},\ldots,\mathbf{o}_{G} \sim \pi_{\theta_{\mathrm{old}}}(\cdot \mid q)$. We extract the executable programs from the responses as $\code_{i}=\operatorname{Code}(\mathbf{o}_{i})$. We execute each program $\code_{i}$ on the same hidden instances $\{T_{m}\}_{m=1}^{M}$. This yields a validity matrix $V\in\{0,1\}^{G\times M}$, where $V_{i,m}=v(\code_{i},T_{m})$. For valid executions, we record the corresponding objective score as $F_{i,m}=f(\code_{i},T_{m})$. For invalid executions, we set $F_{i,m}=-\infty$ as a bookkeeping convention. This gives a raw score matrix $F\in\bar{\mathbb{R}}^{G\times M}$. The shared hidden instances are important because they allow generated programs to be compared under identical evaluation conditions. The supervision signal is therefore not whether a program matches a reference answer, but whether it produces better feasible outputs than other sampled programs on the same instances.

For example, in a TSP-style minimization task, a program reads a cost matrix and outputs a tour. The evaluator first checks whether the tour visits every city exactly once and whether the program finishes within the time limit. If the output is valid, the evaluator computes the tour cost and uses its negation as the objective score. Many programs may be valid on the same instance, but they can receive different scores depending on tour quality. This is the key difference from exact-answer RLVR: the environment supplies relative solution quality without requiring a ground-truth solution.

\subsection{Why Raw Objective Scores Are Not Sufficient}
\label{sec:reward_motivation}

The evaluator score is group-relatively verifiable, but using it directly as a reinforcement learning reward is unstable. A raw-score baseline would first aggregate scores across hidden instances as
$R_i^{\mathrm{raw}}
=
\frac{1}{M}
\sum_{m=1}^{M} F_{i,m}$,
and then compute group-relative advantages from $\{R_i^{\mathrm{raw}}\}_{i=1}^{G}$. This has two failure modes.

First, raw objective scores are not calibrated across instances. Even within the same task, objective magnitudes can vary with instance size, graph structure, weight scale, or outliers. Directly averaging raw scores can make the reward dominated by instances with larger numerical ranges, rather than by robust improvements in relative solution quality. For example, an instance whose scores span thousands of points can outweigh another instance whose scores span only tens of points, even if the latter better separates good and bad programs. This motivates comparing programs within each hidden instance before aggregating feedback across instances.

Second, group-relative updates can over-reinforce frequent non-winning actions. We illustrate this with a simplified one-step policy-gradient toy example. Suppose the sampled group contains one best action $x$ and $n_{y}$ samples of another action type $y$. Here $x$ represents the strongest program discovered in the group, while $y$ represents a frequent but non-winning strategy, such as a common heuristic template with minor syntactic or parameter variations. Let $\widehat{A}{x}$ be the advantage of $x$, and let $\bar{A}{y}$ be the average advantage of the $n_{y}$ samples of type $y$.
Following Eq.~\eqref{eq:policy_gradient}, and ignoring clipping and KL regularization, the update contribution from the rare best action and the aggregate contribution from the frequent non-winning action type can be written as
\begin{equation}
G_{x}
=
\widehat{A}_{x}
\nabla_{\theta}\log \pi_{\theta}(x\mid q),
\quad
G_{y}\approx
n_{y}\bar{A}_{y}
\nabla_{\theta}\log \pi_{\theta}(y\mid q).
\label{eq:frequency_dominance}    
\end{equation}
Thus, even if each individual sample of type $y$ receives a smaller advantage than $x$, the total update mass of $y$ can be larger when $n_{y}\bar{A}_{y}>\widehat{A}_{x}$. Group-relative optimization assigns advantages at the sample level, so a frequently sampled suboptimal strategy can receive more aggregate learning signal than a rare but stronger program.

\subsection{Rank-induced Winner-heavy Rewards}
\label{sec:river_reward}

RiVER constructs rewards in three steps: instance-wise ranking, winner-heavy shaping, and aggregation across hidden instances.

\paragraph{Instance-wise ranking.}
For each hidden instance $T_m$, we rank the $G$ candidate solvers by their objective scores on that instance. Invalid candidates are handled separately and receive the lowest shaped reward. For valid candidates, let $r_{i,m}\in[1,G]$ denote the rank of solver $c_i$ on instance $T_m$, where smaller rank is better.

We use average-tie ranking. If a set of tied candidates occupies rank positions $i,i+1,\ldots,i+k$, then each tied candidate receives rank $i+k/2$. This midrank rule treats tied candidates as occupying the full tied rank interval, rather than assigning all of them the most favorable rank. It therefore reduces over-crediting duplicated solutions with identical execution scores.

Because ranks are computed separately for each hidden instance, they are invariant to instance-dependent score scales. Any strictly monotone transformation of the objective values on the same instance leaves the ranks unchanged. This removes arbitrary score magnitudes while preserving relative solution quality.

\paragraph{Winner-heavy shaping.}
Instance-wise ranks remove scale dependence, but a uniform rank mapping still treats adjacent ranks equally. In open-ended optimization, the most useful signal is often the separation between the best discovered solver and the rest of the group. We therefore use the following shaped reward for each solver-instance pair:
\begin{equation}
s_{i,m}
=
\begin{cases}
-1,
& \text{if } V_{i,m}=0,\\
1,
& \text{if } V_{i,m}=1 \text{ and } r_{i,m}=1,\\
\mathrm{clip}\left((G-1-r_{i,m})/(G-3),0,1\right)-0.5,
& \text{otherwise},
\end{cases}
\label{eq:winner_heavy_reward}
\end{equation}
Thus, an invalid solver receives $-1$, the best valid solver receives $1$, and valid non-winning solvers receive bounded graded feedback in $[-0.5,0.5]$. Ties are handled by the average-rank rule, so duplicated top-scoring candidates do not all receive the same unique-winner credit.

This shaping keeps useful information among non-winning valid solvers while creating a clear margin between the best candidate and the rest. It therefore reduces the chance that repeated suboptimal heuristics dominate the update merely because they appear many times in the rollout group.

\paragraph{Aggregation and GRPO update.}
We average the shaped rewards across hidden instances as $\bar{s}_i=M^{-1}\sum_{m=1}^{M}s_{i,m}$. The resulting scalar is used as the sample-level advantage in GRPO, namely $\widehat{A}_i=\bar{s}_i$. We keep the standard GRPO clipped objective and KL regularization from Eq.~\eqref{eq:grpo_loss}, but replace binary or raw-score rewards with the rank-induced winner-heavy advantage defined above. Since $\bar{s}_i$ is already bounded by construction, we do not apply an additional reward standardization step.

Overall, RiVER converts executable objective feedback into a reward signal with three properties. It is ground-truth-free because it requires no reference solution. It is scale-invariant because candidates are ranked only within the same hidden instance. It is winner-focused but non-binary because the best candidate receives a separated reward while valid non-winners still provide graded feedback.

\section{Experiments}
\label{sec:exp}

\subsection{Experimental Setup}

\paragraph{Training environments.}
We construct the training set from AtCoder Heuristic Contest (AHC) problems released after the ALE-Bench cutoff~\citep{imajuku2025ale,atcoder}, using AHC047--AHC062 as the candidate pool (16 tasks total). We exclude 4 tasks incompatible with our one-pass setting and train on the remaining 12 tasks.\footnote{AHC048--053, AHC055--060.} This ensures no overlap with ALE-Bench evaluation. Each task contains a problem description, an official evaluator, and a test-instance generator. We provide the example prompts in Appendix~\ref{app:prompt}.

\paragraph{AHC scoring.}
Each AHC task is score-based: the official evaluator executes a submitted program and returns a real-valued objective score. For each training prompt, we evaluate 
each candidate solver on 10 hidden test instances and treat the scalar outputs as execution fitness. For minimization tasks, we negate the objective so that 
all tasks follow a unified larger-is-better convention. 
Invalid outputs, runtime errors, and timeouts are treated 
as failures.

\paragraph{Baselines.}
We evaluate two high-performing recent open-source reasoning models, Qwen3-8B~\citep{qwen3technicalreport} and GLM-Z1-9B-0414~\citep{glm2024chatglm}. For each backbone, we compare the original model with several post-training variants that differ in how execution scores are converted into training signals.

We compare against five reward design variants. \textbf{Raw-GRPO}~\cite{shao2024deepseekmath} applies GRPO directly to raw execution scores aggregated over test instances. \textbf{RS-GRPO} replaces standard group normalization with a risk-sensitive transform~\cite{jiang2025risk,yuksekgonul2026learning}. \textbf{Raw-Binary} converts the aggregated raw score into a winner-take-all signal, where the highest-scoring sample receives reward 1 and all others receive 0. \textbf{Instance-Norm} normalizes raw scores within each test instance separately, then averages the per-instance advantages across instances. \textbf{Rank-uniform} uses the same instance-wise ranking as \textbf{RiVER} but assigns uniformly spaced rewards across all ranks. \textbf{RiVER} is our full method, combining instance-wise rank transformation with winner-heavy reward shaping.

\paragraph{Evaluation benchmarks and metrics.}
We evaluate on two types of benchmarks. For score-based evaluation, we use ALE-Bench~\citep{imajuku2025ale}, a benchmark built from AHC problems that evaluates models on optimization-oriented algorithmic tasks without ground-truth solutions. We report \textbf{Rating}, the AtCoder-style aggregate score, and \textbf{Rank\,\%}, the percentile rank among human participants 
induced by that rating, where lower is better. For exact-solution evaluation, we use LiveCodeBench v5 and v6~\citep{jain2024livecodebenchholisticcontaminationfree}, and USACO~\citep{shi2024languagemodelssolveolympiad}, where solutions are judged by exact test-case correctness.  We report the average Pass@1 accuracy across three independent runs for all benchmarks.

\paragraph{Training details.}
We use AdamW~\cite{loshchilov2017decoupled} as the optimizer for all runs with learning rate $1\times10^{-6}$, KL 
coefficient $0.001$, group size $G=16$, and global batch 
size $64$. 

\subsection{Results and Analysis}
\newcommand{\deltagood}[1]{\textcolor{teal}{\scriptsize #1}}
\newcommand{\deltabad}[1]{\textcolor{gray}{\scriptsize #1}}

\begin{table*}[t]
\centering
\resizebox{\textwidth}{!}{%
\begingroup
\renewcommand{\arraystretch}{1.1}
\setlength{\tabcolsep}{3.5pt}

\begin{tabular}{lccccccccccc}
\toprule
\multirow{3}{*}{\textbf{Model}}
& \multicolumn{2}{c}{\textbf{Score-based contests}}
& \multicolumn{9}{c}{\textbf{Exact-solution benchmarks}} \\
\cmidrule(lr){2-3}\cmidrule(lr){4-12}
& \multicolumn{2}{c}{\textbf{ALE-Bench}}
& \multicolumn{4}{c}{\textbf{LiveCodeBench v5}}
& \multicolumn{4}{c}{\textbf{LiveCodeBench v6}}
& \multirow{2}{*}{\textbf{USACO}} \\
\cmidrule(lr){2-3}\cmidrule(lr){4-7}\cmidrule(lr){8-11}
& \textbf{Rating} $\uparrow$
& \textbf{Rank \%} $\downarrow$
& \textbf{Easy}
& \textbf{Med.}
& \textbf{Hard}
& \textbf{Avg.}
& \textbf{Easy}
& \textbf{Med.}
& \textbf{Hard}
& \textbf{Avg.}
&  \\
\midrule

\textbf{Qwen3-8B}
& 845 & 86.4
& 96.8 & 62.9 & 28.8 & 56.1
& 96.3 & 55.6 & 19.7 & 49.2
& 40.4 \\

\quad Raw-GRPO
& 903 & 82.5
& 95.9 & 64.1 & 29.7 & 56.7
& 96.9 & 58.3 & 17.9 & 49.3
& 40.2 \\

\quad RS-GRPO
& 904 & 82.5
& 95.9 & 59.6 & 31.1 & 55.9
& 96.1 & 59.6 & 20.0 & 50.5
& 40.0 \\

\quad Raw-Binary
& 926 & 80.9
& 95.9 & 62.2 & 27.5 & 55.1
& 96.1 & 55.1 & 19.6 & 49.0
& 41.2 \\

\quad Instance-Norm
& 935 & 80.1
& 95.9 & 62.2 & 28.8 & 55.7
& 96.9 & 56.4 & 19.2 & 49.3
& 39.6 \\

\quad Rank-uniform
& 960 & 78.9
& 96.8 & 61.5 & 31.1 & 56.7
& 95.4 & 59.6 & 17.5 & 49.1
& 39.7 \\

\rowcolor[HTML]{F3F6F9}
\quad RiVER
& \textbf{987}
& \textbf{77.5}
& 96.7
& 67.9
& 30.2
& \textbf{58.3}
& 96.9
& 60.3
& 21.2
& \textbf{51.4}
& \textbf{43.3} \\[-3pt]

\rowcolor[HTML]{F3F6F9}
\multicolumn{1}{r}{\scriptsize $\Delta$}
& \deltagood{+142}
& \deltagood{-8.9}
& \deltabad{-0.1}
& \deltagood{+5.0}
& \deltagood{+1.4}
& \deltagood{+2.2}
& \deltagood{+0.6}
& \deltagood{+4.7}
& \deltagood{+1.5}
& \deltagood{+2.2}
& \deltagood{+2.9} \\

\midrule

\textbf{GLM-Z1-9B-0414}
& 805 & 88.2
& 93.4 & 57.7 & 18.6 & 49.2
& 95.6 & 53.3 & 16.1 & 46.7
& 32.4 \\

\quad Raw-GRPO
& 886 & 83.6
& 95.1 & 66.7 & 20.7 & 53.3
& 95.3 & 52.6 & 19.6 & 48.0
& 32.8 \\

\quad RS-GRPO
& 916 & 81.6
& 94.3 & 61.5 & 17.6 & 50.1
& 96.1 & 59.0 & 17.9 & 49.3
& 33.2 \\

\quad Raw-Binary
& 915 & 81.8
& 93.5 & 62.2 & 24.8 & 53.3
& 96.9 & 57.1 & 18.3 & 49.1
& 32.5 \\

\quad Instance-Norm
& 929 & 80.6
& 95.1 & 64.7 & 22.5 & 53.5
& 94.6 & 55.1 & 17.1 & 47.4
& 33.1 \\

\quad Rank-uniform
& 931 & 80.2
& 92.7 & 62.8 & 20.7 & 51.5
& 96.9 & 53.9 & 18.3 & 48.2
& 31.9 \\

\rowcolor[HTML]{F3F6F9}
\quad RiVER
& \textbf{962}
& \textbf{78.8}
& 96.7
& 64.7
& 24.3
& \textbf{54.7}
& 97.7
& 59.0
& 17.9
& \textbf{49.7}
& \textbf{34.3} \\[-3pt]

\rowcolor[HTML]{F3F6F9}
\multicolumn{1}{r}{\scriptsize $\Delta$}
& \deltagood{+157}
& \deltagood{-9.4}
& \deltagood{+3.3}
& \deltagood{+7.0}
& \deltagood{+5.7}
& \deltagood{+5.5}
& \deltagood{+2.1}
& \deltagood{+5.7}
& \deltagood{+1.8}
& \deltagood{+3.0}
& \deltagood{+1.9} \\

\bottomrule
\end{tabular}
\endgroup
}
\vspace{3pt}
\caption{
Results on score-based and pass/fail programming benchmarks.
For score-based contests, we evaluate on ALE-Bench and report the rating and rating percentile rank.
For exact-solution contests, we evaluate on LiveCodeBench v5, LiveCodeBench v6, and USACO.
LiveCodeBench results are further broken down by difficulty.
}
\label{tab:main}
\end{table*}

Table~\ref{tab:main} reports the main results. RiVER 
improves both backbones on every benchmark even though trained on only 12 open-ended optimization tasks without ground truth solutions. Specifically, RiVER raises the ALE rating by 142 points on Qwen3-8B and 157 points on GLM-Z1-9B-0414, and increases performance of the two backbones across LCB v5, LCB v6, and USACO, by 2.4 points and 3.5 points respectively on average. RiVER is also the only method that improves the Qwen3-8B backbone on all five evaluation metrics, where other raw-score or instance-wise ranking baselines have decreased performance on at least one of the exact-solution benchmarks.

\textbf{Raw executable score is not calibrated enough to optimize directly.} Although raw-score baselines such as Raw-GRPO, RS-GRPO, and Raw-Binary increase the ALE rating, the gains fail to transfer to exact-solution coding benchmarks. Specifically, they show no consistent gain on LCB v5, LCB v6, and USACO for the stronger Qwen3-8B backbone; most either stay flat or drop. This suggests that optimizing raw execution scores is likely teaching the model to exploit instance-specific score magnitudes rather than developing transferable coding behavior. 

\textbf{Instance-wise ranking improves score-based 
performance but does not transfer without winner-anchored 
shaping.}
Instance-Norm and Rank-uniform each improve ALE rating 
but show only trivial gains on exact-solution benchmarks. 
Although neither relies on raw execution score magnitudes, 
performance gains on LCB and USACO are close to zero or 
even degrade when applied to Qwen3-8B. On the other hand, adding winner-anchored shaping raises ALE by another 27 points on Qwen3-8B and 31 
on GLM-Z1-9B compared to Rank-uniform, and improves the LCB and USACO by 2.4\% on average. These results suggest that rewarding the winner more than other candidates while still differentiating among them, 
unlike winner-take-all or uniformly spaced rewards, 
can help develop coding skills that transfer to exact-solution tasks.

\paragraph{Difficulty-wise analysis on LiveCodeBench.}
Among all difficulty levels of LiveCodeBench, RiVER improves for both 
backbones, indicating consistent gains rather than 
improvements concentrated in a narrow set of problems. The gains 
are largest on medium and hard problems. For instance, on LiveCodeBench v5, 
Qwen3-8B improves by 5.0\% and 1.4\% on medium and 
hard problems, and GLM-Z1-9B by 7.0\% and 5.7\%. Easy problems 
show little movement, as base models already exceed 93\% 
accuracy on this subset. The concentration of gains on 
harder problems suggests that training on open-ended 
optimization tasks helps the model tackle problems that 
require more sophisticated reasoning, rather than simply 
reinforcing skills that models have already mastered.

\subsection{Case Study}

We present a case study of how reward
design affects the solvers discovered during training. We first examine the
best-so-far training dynamics across all 12 AHC training problems. For each
problem, we evaluate sampled solvers on fixed held-out evaluation instances and
track the best score discovered up to each rollout step. As shown in
Figure~\ref{fig:case_study}, RiVER is competitive across tasks and discovers
stronger or earlier best-so-far solvers on 8 problems. This suggests that the benefit of the
rank-induced reward is not limited to a single task, but appears across multiple
heterogeneous optimization environments.

\begin{figure}
    \centering
    \includegraphics[width=1\linewidth]{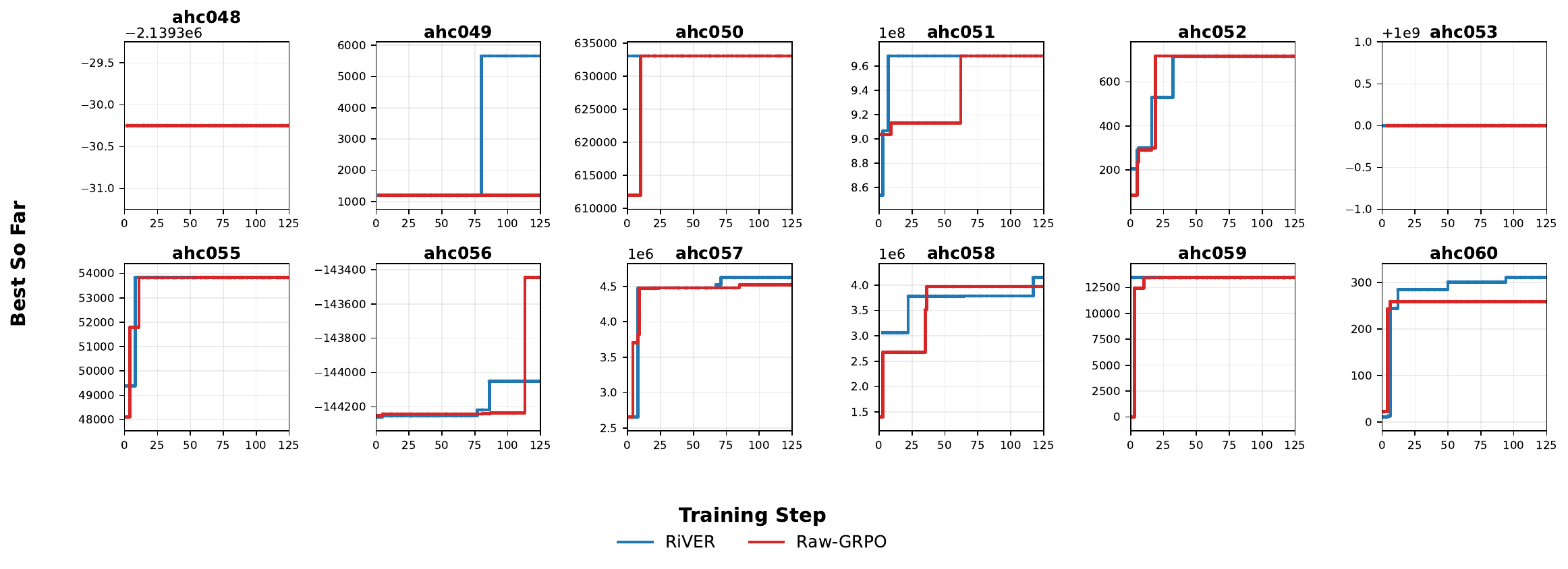}    
    \caption{Best-so-far performance across 12 AHC problems.}
    \label{fig:case_study}
\end{figure}

We then perform a qualitative code inspection on AHC057 to understand what kind
of solver improvements are induced by training. The task description is provided
in Appendix~\ref{app:prompt}, and the generated solvers are provided in
Appendix~\ref{app:ahc057-solvers}.
Before training, the strongest valid base-model rollouts already capture the basic mechanics of AHC057, but use index-based grouping that largely ignores instance geometry. After training, both RiVER and Raw-GRPO shift to static geometric heuristics: they group points by toroidal proximity and connect each group with an MST at $t=0$. RiVER is stronger mainly because its grouping rule is more adaptive: it repeatedly adds the point closest to the current group, whereas Raw-GRPO selects points only by distance to a fixed seed. This produces more compact local clusters and lower within-group connection costs.

\section{Discussion: An Information-Theoretic View of Group-Relative Learning Signals}
\label{sec:gfr}

Our results suggest that score-based optimization tasks are useful not merely because they provide more numerical feedback, but because they expose richer relative distinctions among on-policy samples. In group-relative reinforcement learning, a rollout group is informative only when the sampled responses can be distinguished by the feedback induced by the environment. If all responses in a group receive the same verifier outcome, their relative advantages collapse and the update provides little information about which behavior should be reinforced.

We formalize this intuition using a simple notion of feedback resolution. For a prompt $q$, let
$\mathbf{o}\sim\pi_{\theta_{\mathrm{old}}}(\cdot\mid q)$
be a sampled response, and let
$Z_q = \Phi_q(\mathbf{o})$
denote the feedback induced by a deterministic verifier or executable environment. The feedback $Z_q$ may be binary, scalar-valued, or vector-valued, depending on the environment. In answer-matching RLVR, $Z_q$ often corresponds to a binary correctness signal. In score-based environments, $Z_q$ can represent richer information about solution quality.

\begin{definition}[Group-relative feedback resolution]
\label{def:gfr_main}
Assume the induced feedback distribution is discrete, with
$p_z=\Pr[Z_q=z]$. For a rollout group of size $G>1$, we define the
group-relative feedback resolution of prompt $q$ as the order-$G$ R\'enyi entropy~\citep{renyi1961measures}
\[
\mathrm{GFR}_G(q)
:=
H_G(Z_q)
:=
-\frac{1}{G-1}\log \sum_z p_z^G .
\]
\end{definition}

This quantity measures how much the environment separates on-policy samples into distinguishable feedback values. Binary verifiers have limited feedback resolution, and can be sparse in both low- and high-performance regimes: when the policy is weak, many samples may receive the same failure signal; when the policy is already strong, many samples may receive the same success signal. In both cases, the group contains limited relative information. By contrast, score-based environments can assign multiple feedback levels to different valid or partially valid behaviors, allowing group-relative algorithms to compare candidates by relative quality without requiring a reference solution or an optimal answer.

The connection to group-relative learning can be seen through the probability of feedback collision.

\begin{proposition}[Group feedback collision]
\label{prop:gfr_main}
Let $\mathbf{o}_1,\ldots,\mathbf{o}_G$ be sampled i.i.d. from
$\pi_{\theta_{\mathrm{old}}}(\cdot\mid q)$, and let
$Z_i=\Phi_q(\mathbf{o}_i)$. If the induced feedback distribution is discrete, then
\begin{equation}
    \Pr[Z_1=\cdots=Z_G]
=
\exp\left(-(G-1)\mathrm{GFR}_G(q)\right).
\end{equation}
\end{proposition}

Proposition~\ref{prop:gfr_main} shows that higher feedback resolution reduces the probability that an entire rollout group collapses to a single feedback value. This explains why richer verifiable environments can provide more useful group-relative learning signals than sparse binary verifiers: they are less likely to produce groups in which all samples are indistinguishable under the environment feedback. The proof of Proposition~\ref{prop:gfr_main} is given in Appendix~\ref{app:gfr_proof}.

Feedback resolution alone, however, does not make a reward signal reliable. The ablation results show that dense score feedback improves optimization-oriented ratings more consistently than it transfers to exact-solution benchmarks, suggesting that richness must be paired with calibration. RiVER provides this calibration by converting uncalibrated objective values into instance-wise ranks and applying winner-heavy shaping, preserving relative distinctions while reducing scale dominance and repeated-mode dominance.

This view suggests a practical criterion for constructing ground-truth-free RLVR environments. A useful environment should be verifiable and able to separate on-policy samples into multiple feedback levels. Score-based optimization tasks naturally satisfy this criterion because different candidate behaviors can be compared by objective quality rather than by exact answer matching. RiVER exploits this structure by converting rich but uncalibrated feedback into stable group-relative learning signals. Together with the results in Section~\ref{sec:exp}, this supports the view that score-based optimization tasks can serve as scalable training environments for improving general coding ability without ground-truth solutions.

\section{Conclusion}
We presented RiVER, an RL framework for improving LLMs
in settings where ground-truth answers or optimal
solutions are unavailable. Rather than treating verifiability as exact
answer matching, RiVER exploits a broader class of executable environments
in which candidate solutions can be checked for feasibility and compared by
objective value. 
Empirically, training on AHC tasks
improves performance not only on score-based algorithm engineering
benchmarks, but also on pass/fail programming benchmarks. 
These results suggest that moving beyond answer
matching toward relative, executable, and ground-truth-free verification is
a promising direction for scalable RL of LLMs.

\bibliographystyle{unsrtnat}
\bibliography{main}

\newpage
\clearpage
\appendix

\section{Proof of Group Feedback Collision}
\label{app:gfr_proof}

\begin{proof}[Proof of Proposition~\ref{prop:gfr_main}]
Since $\mathbf{o}_1,\ldots,\mathbf{o}_G$ are sampled independently from
$\pi_{\theta_{\mathrm{old}}}(\cdot\mid q)$, the induced feedback values
$Z_1,\ldots,Z_G$ are i.i.d. with
$\Pr[Z_i=z]=p_z$. Therefore, the event that all sampled responses receive the same feedback value can be decomposed over all possible feedback values:
\[
\Pr[Z_1=\cdots=Z_G]
=
\sum_z
\Pr[Z_1=z,\ldots,Z_G=z].
\]
By independence,
\[
\Pr[Z_1=z,\ldots,Z_G=z]
=
\prod_{i=1}^{G}\Pr[Z_i=z]
=
p_z^G.
\]
Thus,
\[
\Pr[Z_1=\cdots=Z_G]
=
\sum_z p_z^G.
\]
By Definition~\ref{def:gfr_main},
\[
\mathrm{GFR}_G(q)
=
-\frac{1}{G-1}\log\sum_z p_z^G.
\]
Rearranging gives
\[
\sum_z p_z^G
=
\exp\left(-(G-1)\mathrm{GFR}_G(q)\right).
\]
Therefore,
\[
\Pr[Z_1=\cdots=Z_G]
=
\exp\left(-(G-1)\mathrm{GFR}_G(q)\right),
\]
which proves the claim.
\end{proof}

\section{AHC057 Problem Description and Training Prompt}
\label{app:prompt}

This appendix provides the task context used in the AHC057 case study. The
training prompt asks the model to generate a complete Python solver for an
AtCoder Heuristic Contest style optimization problem. Unlike exact-match
programming tasks, this problem is score-based: many solvers may be valid, but
they differ in objective value.

\noindent\textbf{System Prompt}
\begin{promptbox}
You are a top-tier algorithm engineer solving an AtCoder Heuristic Contest style optimization problem.

Your goal is to produce a valid, executable, score-seeking Python solver under strict execution limits.

\medskip
Output Format:

You MUST end your response with a complete, runnable Python code block wrapped in \texttt{python ...} .

After you output the final Python code block, do not output anything else.
\end{promptbox}

\medskip
\noindent\textbf{User Prompt}
\begin{promptbox}

\noindent\textbf{Problem: Story}

AtCoder Laboratory is engaged in the development of new molecules. Takahashi, a genius chemist, can freely bond atoms that move around at any timing he chooses to form molecules. However, bonding atoms consumes energy depending on the distance between them, so he wants to create the required number of molecules using as little energy as possible.

\medskip
\noindent\textbf{Problem Statement}

Consider a two-dimensional plane defined by \(0 \leq x < L = 10^5\) and \(0 \leq y < L = 10^5\). This plane has a toroidal structure, meaning the left and right edges (\(x = 0\) and \(x = L\)), as well as the top and bottom edges (\(y = 0\) and \(y = L\)), are connected. Any coordinates outside this range are normalized to fall within \(0 \leq x,y < L\) by taking the remainder when divided by \(L\) for both \(x\) and \(y\). For example, the position reached by moving \((-20000,30000)\) from \((10000,90000)\) is \((90000,20000)\).

There are \(N\) points on this plane. At time \(t = 0\), the initial position \((x_i,y_i)\) \((0 \leq x_i,y_i < L)\) and initial velocity \((vx_i,vy_i)\) \((-100 \leq vx_i,vy_i \leq 100)\) of the \(i\)-th point \((0 \leq i < N)\) are given as input. In the initial state, each point forms an independent connected component.

At each time \(t = 0,1,\ldots,T-1\), the following two phases are processed in order: 1. Bonding Phase and 2. Movement Phase.

\medskip
\noindent\textbf{1. Bonding Phase}

At time \(t\), it is possible to bond two points that belong to different connected components. Multiple bonds may be performed in the same time step.

Let \((x_i,y_i)\) be the position of point \(i\) at time \(t\), and \((x_j,y_j)\) be the position of point \(j\) at time \(t\). The bonding cost \(D\) for bonding point \(i\) and point \(j\) is calculated using the following distance formula:
\begin{equation}
\tiny
\label{eq:prompt_bond_cost}
\begin{aligned}
D
&=
\mathrm{round}
\left(
\sqrt{
\bigl(\min(L-\Delta x,\Delta x)\bigr)^2
+
\bigl(\min(L-\Delta y,\Delta y)\bigr)^2
}
\right),
\\
\Delta x &= |x_i-x_j|,
\qquad
\Delta y = |y_i-y_j|.
\end{aligned}
\end{equation}

When two points are bonded, the velocity of the resulting connected component is updated according to the law of conservation of momentum. Suppose that before the bond, point \(i\) belongs to connected component \(A\) moving at velocity \((vx_A,vy_A)\), and point \(j\) belongs to connected component \(B\) moving at velocity \((vx_B,vy_B)\). Then, for all points belonging to connected components \(A\) and \(B\), the velocity \((vx_{\text{new}},vy_{\text{new}})\) after bonding is updated as follows:
\begin{equation}
\label{eq:prompt_velocity_update}
\begin{aligned}
vx_{\text{new}}
&=
\frac{|A|\times vx_A + |B|\times vx_B}
{|A|+|B|},
\\
vy_{\text{new}}
&=
\frac{|A|\times vy_A + |B|\times vy_B}
{|A|+|B|}.
\end{aligned}
\end{equation}

Here, \(|A|\) and \(|B|\) denote the number of points in connected components \(A\) and \(B\), respectively. After bonding, all points in the resulting connected component will move with the same velocity. The positions of the points do not change due to bonding. Additionally, the order in which bonds are performed within the same time step does not affect the resulting positions, bonding costs, or velocities of the connected components.

Coordinates and velocities may become fractional, but all computations are performed using double-precision floating-point arithmetic.

\medskip
\noindent\textbf{1. Movement Phase}

The positions of all points in each connected component are updated simultaneously. Let \((x_i,y_i)\) be the position of point \(i\) at time \(t\), and let \((vx_i,vy_i)\) be the velocity of the connected component to which point \(i\) belongs. Then, the position \((x_i',y_i')\) of point \(i\) at time \(t+1\) is updated as follows:
\begin{equation}
\label{eq:prompt_movement}
\begin{aligned}
x_i' &= (x_i + vx_i) \bmod L,
\\
y_i' &= (y_i + vy_i) \bmod L.
\end{aligned}
\end{equation}

Plan the bonds so that at time \(T\), the number of connected components is exactly \(M\), and the size of each connected component is exactly \(K\), while minimizing the total bonding cost \(D\).

\medskip
\noindent\textbf{Example of Bonding and Movement}

The figure above illustrates an example of bonding three points into a single connected component. First, the two points at the lower left and lower right are bonded, and their direction of movement changes upward. Next, they are bonded with the point moving to the right at the top, and the direction of movement changes to upward-right.

\medskip
\noindent\textbf{Scoring}

Let the total bonding cost of all bonds be \(D_{\text{sum}}\). The score for a single test case is calculated as follows. A higher score is better.
\begin{equation}
\label{eq:prompt_score}
\mathrm{Score}
=
\mathrm{round}
\left(
10^6
\times
\log_2
\left(
\frac{L\times(N-M)}
{D_{\text{sum}}+1}
\right)
\right).
\end{equation}

The result will be WA in the following cases:
\begin{itemize}[leftmargin=1.5em,itemsep=0.2em]
    \item If the output is invalid
    \item If at time \(T\), the number of connected components is not exactly \(M\), or the size of each connected component is not exactly \(K\)
    \item If a bond is specified between two points that already belong to the same connected component
\end{itemize}

There are 150 test cases, and the score of a submission is the total score for each test case. If your submission produces an illegal output or exceeds the time limit for some test cases, the submission itself will be judged as WA or TLE, and the score of the submission will be zero. The highest score obtained during the contest will determine the final ranking, and there will be no system test after the contest. If more than one participant gets the same score, they will be ranked in the same place regardless of the submission time.

\medskip
\noindent\textbf{Benchmark Interface}

In this benchmark, the official contest input is already parsed before your solver is called. Implement the typed \texttt{solve(...)->List[str]} signature shown below instead of reading from standard input.

Solver arguments:
\begin{itemize}[leftmargin=1.5em,itemsep=0.2em]
    \item \texttt{point\_count}: \(N\)
    \item \texttt{time\_steps}: \(T\)
    \item \texttt{target\_components}: \(M\)
    \item \texttt{target\_component\_size}: \(K\)
    \item \texttt{torus\_size}: \(L\)
    \item \texttt{points}: the full list of \([x_i,y_i,vx_i,vy_i]\) records
\end{itemize}

The parsed instance still satisfies the official constraints:
\begin{itemize}[leftmargin=1.5em,itemsep=0.2em]
    \item The total number of points is \(N = 300\).
    \item The total number of steps is \(T = 1000\).
    \item The target number of connected components is \(M = 10\).
    \item The target size of each connected component is \(K = 30\).
    \item The side length of the space is \(L = 10^5\).
    \item Each position satisfies \(0 \leq x_i,y_i < L\).
    \item Each velocity satisfies \(100 \leq vx_i,vy_i \leq 100\).
\end{itemize}

Return value:
\begin{itemize}[leftmargin=1.5em,itemsep=0.2em]
    \item Return the official output as \texttt{List[str]}, where each element is one stdout line.
    \item Output exactly \texttt{point\_count - target\_components} bond operations.
    \item Each output line must contain \texttt{t i j}, meaning that points \(i\) and \(j\) are bonded at time \(t\).
    \item The evaluator processes bonds in ascending order of time \(t\).
    \item Every output line must satisfy \(0 \leq t < \texttt{time\_steps}\), \(0 \leq i,j < \texttt{point\_count}\), and \(i \neq j\).
\end{itemize}

\medskip
\noindent\textbf{Input Generation} $\mathrm{rand}(L,U)$:
Randomly generates an integer between $L$ and $U$,
 inclusive, with uniform probability.

\medskip
\noindent\textbf{Generating Initial Positions}

For each point \(i\), independently generate
\begin{equation}
\label{eq:prompt_init_pos}
\begin{aligned}
x_i &= \mathrm{rand}(0,L-1),
\\
y_i &= \mathrm{rand}(0,L-1).
\end{aligned}
\end{equation}

\medskip
\noindent\textbf{Generating Initial Velocities}

For each point \(i\), independently generate
\begin{equation}
\label{eq:prompt_init_vel}
\begin{aligned}
vx_i &= \mathrm{rand}(-100,100),
\\
vy_i &= \mathrm{rand}(-100,100).
\end{aligned}
\end{equation}

\medskip
\noindent\textbf{Tools (Input generator and visualizer)}

\begin{itemize}[leftmargin=1.5em,itemsep=0.2em]
    \item Web version: This is more powerful than the local version providing animations.
    \item Local version: You need a compilation environment of Rust language.
    \item Pre-compiled binary for Windows: If you are not familiar with the Rust language environment, please use this instead.
\end{itemize}

Please be aware that sharing visualization results or discussing solutions/ideas during the contest is prohibited.

\medskip
\noindent\textbf{Benchmark Interface:}

\begin{itemize}[leftmargin=1.5em,itemsep=0.2em]
    \item The official contest input has already been parsed into the typed arguments of the \texttt{solve(...)->List[str]} function shown below.
    \item Use those arguments directly instead of reading from standard input.
    \item Return the solver output as \texttt{List[str]}, where each element is one stdout line.
    \item Do not call \texttt{input()}, \texttt{sys.stdin.read()}, or use print-driven main logic.
\end{itemize}

\medskip
\noindent\textbf{Constraints:}

\begin{enumerate}[leftmargin=1.5em,itemsep=0.2em]
    \item Execution time must be under 5 seconds.
    \item Your code will be scored by the quality of the solution it produces. Higher-quality solutions receive higher rewards.
    \item Use the specific function signature provided below.
    \item The returned output must strictly follow the official output format.
\end{enumerate}

\begin{lstlisting}[style=appendixcode]
from typing import List

def solve(
    point_count: int,
    time_steps: int,
    target_components: int,
    target_component_size: int,
    torus_size: int,
    points: List[List[int]],
) -> List[str]:
    """
    Args:
        point_count: Number of points `N`.
        time_steps: Total number of time steps `T`.
        target_components: Required component count `M`.
        target_component_size: Required component size `K`.
        torus_size: Side length `L` of the torus.
        points: `[x_i, y_i, vx_i, vy_i]` records for all points.

    Returns:
        Output lines exactly as they should appear on stdout.
        Each element of the returned list is one line.
    """
\end{lstlisting}

\end{promptbox}

\section{AHC057 Solver Listings}
\label{app:ahc057-solvers}

This appendix lists the representative Python solvers analyzed in the AHC057 case study. Each solver follows the interface
\texttt{solve(point\_count, time\_steps, target\_components, target\_component\_size, torus\_size, points) -> List[str]}.

\subsection{RiVER-Best Solver}
\label{app:average-best-code}

\begin{lstlisting}[style=appendixcode, caption={RiVER-best AHC057 solver.}]
import math
from typing import List

def solve(
    point_count: int,
    time_steps: int,
    target_components: int,
    target_component_size: int,
    torus_size: int,
    points: List[List[int]],
) -> List[str]:
    def compute_min_distance(x1, y1, x2, y2, L):
        dx = abs(x1 - x2)
        dx = min(dx, L - dx)
        dy = abs(y1 - y2)
        dy = min(dy, L - dy)
        return math.sqrt(dx*dx + dy*dy)
    
    available = list(range(point_count))
    groups = []
    
    for _ in range(target_components):
        if not available:
            break
        current_group = [available[0]]
        available.pop(0)
        while len(current_group) < target_component_size:
            min_dist = float('inf')
            best_idx = -1
            for candidate in available:
                current_min = float('inf')
                for point in current_group:
                    d = compute_min_distance(points[point][0], points[point][1], points[candidate][0], points[candidate][1], torus_size)
                    if d < current_min:
                        current_min = d
                if current_min < min_dist:
                    min_dist = current_min
                    best_idx = candidate
            if best_idx == -1:
                break
            current_group.append(best_idx)
            available.remove(best_idx)
        groups.append(current_group)
    
    bonds = []
    for group in groups:
        edges = []
        for i in range(len(group)):
            for j in range(i + 1, len(group)):
                p1 = group[i]
                p2 = group[j]
                x1, y1 = points[p1][0], points[p1][1]
                x2, y2 = points[p2][0], points[p2][1]
                dx = abs(x1 - x2)
                dx = min(dx, torus_size - dx)
                dy = abs(y1 - y2)
                dy = min(dy, torus_size - dy)
                dist = math.sqrt(dx*dx + dy*dy)
                edges.append((dist, p1, p2))
        edges.sort()
        parent = {node: node for node in group}
        def find(u):
            while parent[u] != u:
                parent[u] = parent[parent[u]]
                u = parent[u]
            return u
        def union(u, v):
            root_u = find(u)
            root_v = find(v)
            if root_u != root_v:
                parent[root_v] = root_u
        mst_edges = []
        for dist, u, v in edges:
            if find(u) != find(v):
                union(u, v)
                mst_edges.append((u, v))
        for u, v in mst_edges:
            bonds.append((0, u, v))
    
    return [f"0 {u} {v}" for (t, u, v) in bonds]
\end{lstlisting}

\subsection{Raw-GRPO-Best Solver}
\label{app:grpo-best-code}

\begin{lstlisting}[style=appendixcode, caption={Raw-GRPO-best AHC057 solver.}]
import math
from typing import List

def solve(
    point_count: int,
    time_steps: int,
    target_components: int,
    target_component_size: int,
    torus_size: int,
    points: List[List[int]],
) -> List[str]:
    """
    Args:
        point_count: Number of points `N`.
        time_steps: Total number of time steps `T`.
        target_components: Required component count `M`.
        target_component_size: Required component size `K`.
        torus_size: Side length `L` of the torus.
        points: `[x_i, y_i, vx_i, vy_i]` records for all points.

    Returns:
        Output lines exactly as they should appear on stdout.
        Each element of the returned list is one line.
    """
    if target_components * target_component_size != point_count:
        return []

    M = target_components
    K = target_component_size
    L = torus_size

    remaining_points = list(range(point_count))
    groups = []

    # Group points into M groups of K points each
    for _ in range(M):
        # Select the first point as seed
        seed = remaining_points[0]
        group = [seed]
        remaining_points.pop(0)

        # Select K-1 points closest to the seed
        for _ in range(K - 1):
            if not remaining_points:
                break
            min_dist = float('inf')
            min_idx = -1
            for idx in remaining_points:
                x1, y1 = points[seed][0], points[seed][1]
                x2, y2 = points[idx][0], points[idx][1]
                dx = abs(x1 - x2)
                dx = min(dx, L - dx)
                dy = abs(y1 - y2)
                dy = min(dy, L - dy)
                dist = math.sqrt(dx**2 + dy**2)
                if dist < min_dist:
                    min_dist = dist
                    min_idx = idx
            if min_idx != -1:
                group.append(min_idx)
                remaining_points.remove(min_idx)
        groups.append(group)

    # Function to generate MST edges for a group
    def generate_mst_edges(group):
        edges = []
        for i in range(len(group)):
            for j in range(i + 1, len(group)):
                x1, y1 = points[group[i]][0], points[group[i]][1]
                x2, y2 = points[group[j]][0], points[group[j]][1]
                dx = abs(x1 - x2)
                dx = min(dx, L - dx)
                dy = abs(y1 - y2)
                dy = min(dy, L - dy)
                dist = math.sqrt(dx**2 + dy**2)
                edges.append((dist, group[i], group[j]))
        # Kruskal's algorithm
        parent = {}
        for node in group:
            parent[node] = node

        def find(u):
            while parent[u] != u:
                parent[u] = parent[parent[u]]
                u = parent[u]
            return u

        def union(u, v):
            root_u = find(u)
            root_v = find(v)
            if root_u != root_v:
                parent[root_v] = root_u
                return True
            return False

        edges.sort()
        mst_edges = []
        for dist, u, v in edges:
            if find(u) != find(v):
                union(u, v)
                mst_edges.append((u, v))
        return mst_edges

    # Collect all bonds
    bonds = []
    for group in groups:
        mst_edges = generate_mst_edges(group)
        for u, v in mst_edges:
            bonds.append((0, u, v))

    # Generate output
    output = []
    for t, u, v in bonds:
        output.append(f"{t} {u} {v}")
    return output
\end{lstlisting}

\subsection{Qwen3-8B Solver}
\label{app:grpo-first-code}

\begin{lstlisting}[style=appendixcode, caption={Qwen3-8B AHC057 solver.}]
from typing import List

def solve(
    point_count: int,
    time_steps: int,
    target_components: int,
    target_component_size: int,
    torus_size: int,
    points: List[List[int]],
) -> List[str]:
    """
    Args:
        point_count: Number of points N.
        time_steps: Total number of time steps T.
        target_components: Required component count M.
        target_component_size: Required component size K.
        torus_size: Side length L of the torus.
        points: `[x_i, y_i, vx_i, vy_i]` records for all points.

    Returns:
        Output lines exactly as they should appear on stdout.
        Each element of the returned list is one line.
    """
    M = target_components
    K = target_component_size
    L = torus_size

    # Preassign groups
    groups = [[] for _ in range(M)]
    for i in range(point_count):
        group = i // K
        groups[group].append(i)

    # Initialize components
    components = []
    for i in range(point_count):
        x = points[i][0]
        y = points[i][1]
        vx = points[i][2]
        vy = points[i][3]
        group_id = i // K
        components.append({
            'points': [i],
            'x': x,
            'y': y,
            'vx': vx,
            'vy': vy,
            'group': group_id
        })

    bonds = []

    for time in range(time_steps):
        # Bonding phase
        for group_id in range(M):
            # Collect all components in this group
            group_components = [c for c in components if c['group'] == group_id]
            pairs = []
            for i in range(len(group_components)):
                for j in range(i + 1, len(group_components)):
                    a = group_components[i]
                    b = group_components[j]
                    dx = abs(a['x'] - b['x'])
                    dy = abs(a['y'] - b['y'])
                    min_dx = min(L - dx, dx)
                    min_dy = min(L - dy, dy)
                    distance = (min_dx**2 + min_dy**2)**0.5
                    pairs.append((distance, a, b))
            if pairs:
                # Sort pairs by distance
                pairs.sort()
                min_distance, a, b = pairs[0]
                # Record bond between first point of a and first point of b
                i = a['points'][0]
                j = b['points'][0]
                bonds.append((time, i, j))
                # Merge a and b
                new_points = a['points'] + b['points']
                new_x = a['x']
                new_y = a['y']
                new_vx = (a['vx'] * len(a['points']) + b['vx'] * len(b['points'])) / len(new_points)
                new_vy = (a['vy'] * len(a['points']) + b['vy'] * len(b['points'])) / len(new_points)
                new_group = a['group']
                # Update components list
                new_components = []
                for c in components:
                    if c is a or c is b:
                        continue
                    new_components.append(c)
                new_components.append({
                    'points': new_points,
                    'x': new_x,
                    'y': new_y,
                    'vx': new_vx,
                    'vy': new_vy,
                    'group': new_group
                })
                components = new_components

        # Movement phase
        for c in components:
            c['x'] = (c['x'] + c['vx']) %
            c['y'] = (c['y'] + c['vy']) %

    # Format the bonds
    bond_output = []
    for t, i, j in bonds:
        bond_output.append(f"{t} {i} {j}")

    return bond_output
\end{lstlisting}

\end{document}